\documentclass[runningheads]{llncs}
\usepackage{graphicx}
\usepackage{amsmath,amssymb} 
\usepackage{color}

\usepackage{algorithm}
\usepackage{algorithmic}
\usepackage{url}
\usepackage{epigraph}
\usepackage{enumitem}
\usepackage{color}
\usepackage{soul}
\usepackage{subfigure}
\usepackage{mdframed}

\newcommand{\qqed}{\hfill$\square$}

\newcommand{\bx}{\mathbf{x}}

\newcommand{\ba}{\mathbf{a}}

\newcommand{\bt}{\mathbf{t}}

\newcommand{\cP}{\mathcal{P}}

\newcommand{\cQ}{\mathcal{Q}}

\newcommand{\cG}{\mathcal{G}}

\newcommand{\cB}{\mathcal{B}}

\newcommand{\cI}{\mathcal{I}}

\newcommand{\cD}{\mathcal{D}}

\newcommand{\cO}{\mathcal{O}}

\newcommand{\cC}{\mathcal{C}}

\newcommand{\cK}{\mathcal{K}}

\newcommand{\bbI}{\mathbb{I}}

\newcommand{\LP}{\mathrm{LP}}
\newcommand{\poly}{\mathrm{poly}}

\def\med{\operatorname*{med\,}}

\begin{document}
\title{Robust Fitting in Computer Vision:\\Easy or Hard?} 

\titlerunning{Robust Fitting in Computer Vision: Easy or Hard?} 
%
\author{Tat-Jun Chin \and
Zhipeng Cai \and
Frank Neumann}
%
\index{Chin, Tat-Jun}
\authorrunning{T.-J. Chin, Z. Cai, and F. Neumann}
%

\institute{School of Computer Science, The University of Adelaide}
\maketitle              
\begin{abstract}
Robust model fitting plays a vital role in computer vision, and research into algorithms for robust fitting continues to be active. Arguably the most popular paradigm for robust fitting in computer vision is \emph{consensus maximisation}, which strives to find the model parameters that maximise the number of inliers. Despite the significant developments in algorithms for consensus maximisation, there has been a lack of fundamental analysis of the problem in the computer vision literature. In particular, whether consensus maximisation is ``tractable" remains a question that has not been rigorously dealt with, thus making it difficult to assess and compare the performance of proposed algorithms, relative to what is theoretically achievable. To shed light on these issues, we present several computational hardness results for consensus maximisation. Our results underline the fundamental intractability of the problem, and resolve several ambiguities existing in the literature.

\end{abstract}

\section{Introduction}

Robustly fitting a geometric model onto noisy and outlier-contaminated data is a necessary capability in computer vision~\cite{meer04}, due to the imperfectness of data acquisition systems and preprocessing algorithms (e.g., edge detection, keypoint detection and matching). Without robustness against outliers, the estimated geometric model will be biased, leading to failure in the overall pipeline.

In computer vision, robust fitting is typically performed under the framework of \emph{inlier set maximisation}, a.k.a.~\emph{consensus maximisation}~\cite{fischler81}, where one seeks the model with the most number of inliers. For concreteness, say we wish to estimate the parameter vector $\bx \in \mathbb{R}^d$ that defines the linear relationship $\ba^T \bx = b$ from a set of outlier-contaminated measurements $\cD = \{(\ba_i,b_i)\}^{N}_{i=1}$. The consensus maximisation formulation for this problem is as follows.

\begin{problem}[MAXCON]
Given input data $\cD = \{(\ba_i,b_i)\}^{N}_{i=1}$, where $\ba_i \in \mathbb{R}^d$ and $b_i \in \mathbb{R}$, and an inlier threshold $
\epsilon \in \mathbb{R}_+$, find the $\bx \in \mathbb{R}^d$ that maximises
\begin{align}\label{equ:consensus}
\mathrm{\Psi}_\epsilon(\bx \mid \cD) = \sum_{i=1}^N \bbI\left( |\ba_i^T\bx - b_i | \le \epsilon \right),
\end{align}
where $\bbI$ returns $1$ if its input predicate is true, and $0$ otherwise.
\end{problem}

The quantity $| \ba_i^T\bx - b_i |$ is the \emph{residual} of the $i$-th measurement with respect to $\bx$, and the value given by $\mathrm{\Psi}_\epsilon(\bx \mid \cD)$ is the \emph{consensus} of $\bx$ with respect to $\cD$. Intuitively, the consensus of $\bx$ is the number of inliers of $\bx$. For the robust estimate to fit the inlier structure well, the inlier threshold $\epsilon$ must be set to an appropriate value; the large number of applications that employ the consensus maximisation framework indicate that this is usually not an obstacle.

Developing algorithms for robust fitting, specifically for consensus maximisation, is an active research area in computer vision. Currently, the most popular algorithms belong to the class of randomised sampling techniques, i.e., RANSAC~\cite{fischler81} and its variants~\cite{choi09,raguram13}. Unfortunately, such techniques do not provide certainty of finding satisfactory solutions, let alone optimal ones~\cite{tran14}.

Increasingly, attention is given to constructing \emph{globally optimal} algorithms for robust fitting, e.g.,~\cite{li09,zheng11,enqvist12,bazin13,yang14,parrabustos14,enqvist15,chin15,campbell17}. Such algorithms are able to deterministically calculate the best possible solution, i.e., the model with the highest achievable consensus. This mathematical guarantee is regarded as desirable, especially in comparison to the ``rough" solutions provided by random sampling heuristics.

Recent progress in globally optimal algorithms for consensus maximisation seems to suggest that global solutions can be obtained efficiently or tractably~\cite{li09,zheng11,enqvist12,bazin13,yang14,parrabustos14,enqvist15,chin15,campbell17}. Moreover, decent empirical performances have been reported. This raises hopes that good alternatives to the random sampling methods are now available. However, to what extent is the problem solved? Can we expect the global algorithms to perform well in general? Are there fundamental obstacles toward efficient robust fitting algorithms? What do we even mean by ``efficient"?

\subsection{Our contributions and their implications}\label{sec:contributions}

Our contributions are \emph{theoretical}. We resolve the above ambiguities in the literature, by proving the following computational hardness results. The implications of each result are also listed below.
\begin{mdframed}
MAXCON is NP-hard (Section~\ref{sec:nphard}).
\end{mdframed}
\begin{itemize}[parsep=0em,topsep=0em]
\item[$\implies$] There are no algorithms that can solve MAXCON in time polynomial to the input size, which is proportional to $N$ and $d$.
\end{itemize}
\begin{mdframed}
MAXCON is W[1]-hard in the dimension $d$ (Section~\ref{sec:w1hard}).
\end{mdframed}
\begin{itemize}[parsep=0em,topsep=0em]
\item[$\implies$] There are no algorithms that can solve MAXCON in time $f(d)\mathrm{poly}(N)$, where $f(d)$ is an arbitrary function of $d$, and $\mathrm{poly}(N)$ is a polynomial of $N$.
\end{itemize}
\begin{mdframed}
MAXCON is APX-hard (Section~\ref{sec:apxhard}).
\end{mdframed}
\begin{itemize}[parsep=2pt,topsep=2pt]
\item[$\implies$] There are no polynomial time algorithms that can approximate MAXCON up to $(1-\delta)\psi^*$ for any known factor $\delta$, where $\psi^\ast$ is the maximum consensus.
\end{itemize}
As usual, the implications of the hardness results are subject to the standard complexity assumptions P$\ne$NP~\cite{garey90} and FPT$\ne$W[1]-hard~\cite{downey99}.

Our analysis indicates the ``extreme" difficulty of consensus maximisation. MAXCON is not only \emph{intractable} (by standard notions of intractability~\cite{garey90,downey99}), the W[1]-hardness result also suggests that any global algorithm will scale exponentially in a function of $d$, i.e., $N^{f(d)}$. In fact, if a conjecture of Erickson et al.~\cite{erickson06} holds, MAXCON cannot be solved faster than $N^d$. Thus, the decent performances in~\cite{li09,zheng11,enqvist12,bazin13,yang14,parrabustos14,enqvist15,chin15,campbell17} are unlikely to extend to the general cases in practical settings, where $N \ge 1000$ and $d \ge 6$ are common. More pessimistically, APX-hardness shows that MAXCON is impossible to approximate, in that there are no polynomial time approximation schemes (PTAS)~\cite{vazirani01} for MAXCON\footnote{Since RANSAC does not provide any approximation guarantees, it is not an ``approximation scheme" by standard definition~\cite{vazirani01}.}. 

A slightly positive result is as follows.
\begin{mdframed}
MAXCON is FPT (fixed parameter tractable) in the number of outliers $o$ and dimension $d$ (Section~\ref{sec:fpt}).
\end{mdframed}
This is achieved by applying a special case of the algorithm of Chin et al.~\cite{chin15} on MAXCON to yield a runtime of $\cO(d^o)\textrm{poly}(N,d)$. However, this still scales exponentially in $o$, which can be large in practice (e.g., $o \ge 100$).

\subsection{How are our theoretical results useful?}

First, our results clarify the ambiguities on the efficiency and solvability of consensus maximisation alluded to above. Second, our analysis shows how the effort scales with the different input size parameters, thus suggesting more cogent ways for researchers to test/compare algorithms. Third, since developing algorithms for consensus maximisation is an active topic, our hardness results encourage researchers to consider alternative paradigms of optimisation, e.g., deterministically convergent heuristic algorithms~\cite{le17,purkait17,cai18} or preprocessing techniques~\cite{svarm14,parrabustos15,chin16}. 


\subsection{What about non-linear models?}

Our results are based specifically on MAXCON, which is concerned with fitting linear models. In practice, computer vision applications require the fitting of non-linear geometric models (e.g., fundamental matrix, homography, rotation). While a case-by-case treatment is ideal, it is unlikely that non-linear consensus maximisation will be easier than linear consensus maximisation~\cite{johnson78,ben-david02,aronov08}.

\subsection{Why not employ other robust statistical procedures?}

Our purpose here is not to benchmark or advocate certain robust criteria. Rather, our primary aim is to establish the fundamental difficulty of consensus maximisation, which is widely used in computer vision. Second, it is unlikely that other robust criteria are easier to solve~\cite{bernholt06}. Although some that use differentiable robust loss functions (e.g., M-estimators) can be solved up to local optimality, it is unknown how far the local optima deviate from the global solution.

The rest of the paper is devoted to developing the above hardness results.

\section{NP-hardness}\label{sec:nphard}

The decision version of MAXCON is as follows.
\begin{problem}[MAXCON-D]\label{prob:maxcon}
Given data $\cD = \{(\ba_i,b_i)\}^{N}_{i=1}$, an inlier threshold $\epsilon \in \mathbb{R}_+$, and a number $\psi \in \mathbb{N}_+$, does there exist $\bx \in \mathbb{R}^d$ such that $\Psi_\epsilon(\bx \mid \cD) \ge \psi$?
\end{problem} 

Another well-known robust fitting paradigm is least median squares (LMS), where we seek the vector $\bx$ that minimises the median of the residuals
\begin{align}\label{equ:lms}
\underset{\bx \in \mathbb{R}^d}{\min} \hspace{1em} \med\left( |\ba_1^T\bx - b_1|, \dots, |\ba_N^T\bx - b_N| \right).
\end{align}
LMS can be generalised by minimising the $k$-th largest residual instead
\begin{align}\label{equ:lkos}
\underset{\bx \in \mathbb{R}^d}{\min} \hspace{1em} \mathrm{kos}\left( |\ba_1^T\bx - b_1|, \dots, |\ba_N^T\bx - b_N| \right),
\end{align}
where function $\mathrm{kos}$ returns its $k$-th largest input value.

Geometrically, LMS seeks the \emph{slab} of the \emph{smallest width} that contains \emph{half} of the data points $\cD$ in $\mathbb{R}^{d+1}$. A slab in $\mathbb{R}^{d+1}$ is defined by a normal vector $\bx$ and width $w$ as
\begin{align}
h_w(\bx) = \left\{ (\ba,b) \in \mathbb{R}^{d+1} \; \left| \; |\ba^T\bx-b | \le \frac{1}{2}w \right. \right\}.
\end{align}
Problem~\eqref{equ:lkos} thus seeks the thinnest slab that contains $k$ of the points. The decision version of~\eqref{equ:lkos} is as follows.

\begin{problem}[k-SLAB]\label{prob:kslab}
Given data $\cD = \{(\ba_i,b_i)\}^{N}_{i=1}$, an integer $k$ where $1 \le k \le N$, and a number $w^\prime \in \mathbb{R}_+$, does there exist $\bx \in \mathbb{R}^d$ such that $k$ of the members of $\cD$ are contained in a slab $h_w(\bx)$ of width at most $w^\prime$?
\end{problem}
k-SLAB has been proven to be NP-complete in~\cite{erickson06}.

\begin{theorem}
MAXCON-D is NP-complete.
\end{theorem}
\begin{proof}
Let $\cD$, $k$ and $w^\prime$ define an instance of k-SLAB. This can be reduced to an instance of MAXCON-D by simply reusing the same $\cD$, and setting $\epsilon = \frac{1}{2}w^\prime$ and $\psi = k$. If the answer to k-SLAB is positive, then there is an $\bx$ such that $k$ points from $\cD$ lie within vertical distance of $\frac{1}{2}w^\prime$ from the hyperplane defined by $\bx$, hence $\Psi_\epsilon(\bx \mid \cD)$ must be at least $\psi$ and the answer to MAXCON-D is also positive. Conversely, if the answer to MAXCON-D is positive, then there is an $\bx$ such that $\psi$ points have vertical distance of less than $\epsilon$ to $\bx$, hence a slab that is centred at $\bx$ of width at most $w^\prime$ can enclose $k$ of the points, and the answer to k-SLAB is also positive.\qqed
\end{proof}

The NP-completeness of MAXCON-D implies the NP-hardness of the optimisation version MAXCON. See Sec.~\ref{sec:contributions} for the implications of NP-hardness.

\section{Parametrised complexity}

Parametrised complexity is a branch of algorithmics that investigates the inherent difficulty of problems with respect to structural parameters in the input~\cite{downey99}. In this section, we report several parametrised complexity results of MAXCON.

First, the \emph{consensus set} $\cC_\epsilon(\bx \mid \cD)$ of $\bx$ is defined as
\begin{align}
\cC_\epsilon(\bx \mid \cD) := \{ i \in \{1,\dots,N \} \mid  | \ba^T_i \bx - b_i | \le \epsilon \}.
\end{align}
An equivalent definition of consensus~\eqref{equ:consensus} is thus
\begin{align}
\mathrm{\Psi}_\epsilon(\bx \mid \cD) = |\cC_\epsilon(\bx \mid \cD)|.
\end{align}
Henceforth, we do not distinguish between the integer subset $\cC \subseteq \{1,\dots,N \}$ that indexes a subset of $\cD$, and the actual data that are indexed by $\cC$.

\subsection{XP in the dimension}

The following is the \emph{Chebyshev approximation} problem~\cite[Chapter 2]{cheney66} defined on the input data indexed by $\cC$:
\begin{align}\label{equ:chebyshev}
\underset{\bx \in \mathbb{R}^d}{\min} \hspace{1em} \max _{i \in \cC}~|\ba_i^T\bx - b_i|
\end{align}
Problem~\eqref{equ:chebyshev} has the linear programming (LP) formulation
\begin{align}\tag{$\LP[\cC]$}\label{equ:lp}
\begin{aligned}
& \underset{\bx \in \mathbb{R}^d, \gamma \in \mathbb{R}}{\min}
& & \gamma \\
& \text{s.t.}
& & |\ba_i^T\bx - b_i| \le \gamma, \; i \in \cC,
\end{aligned}
\end{align}
which can be solved in polynomial time. Chebyshev approximation also has the following property.

\begin{lemma}\label{lem:basis}
There is a subset $\cB$ of $\cC$, where $|\cB| \le d+1$, such that
\begin{align}\label{equ:basis}
\underset{\bx \in \mathbb{R}^d}{\min} \hspace{1em} \max _{i \in \cB}~r_i(\bx) = \underset{\bx \in \mathbb{R}^d}{\min} \hspace{1em} \max _{i \in \cC}~r_i(\bx)
\end{align}
\end{lemma}
\begin{proof}
See~\cite[Section 2.3]{cheney66}. \qqed
\end{proof}

We call $\cB$ a \emph{basis} of $\cC$. Mathematically, $\cB$ is the set of active constraints to $\LP[\cC]$, hence bases can be computed easily. In fact, $\LP[\cB]$ and $\LP[\cC]$ have the same minimisers. Further, for any subset $\cB$ of size $d+1$, a method by de la Vall\'{e}e-Poussin can solve $\LP[\cB]$ analytically in time polynomial to $d$; see~\cite[Chapter 2]{cheney66} for details.

Let $\bx$ be an arbitrary candidate solution to MAXCON, and $(\hat{\bx}, \hat{\gamma})$ be the minimisers to $\LP[\cC_\epsilon(\bx \mid \cD)]$, i.e., the Chebyshev approximation problem on the consensus set of $\bx$. The following property can be established.

\begin{lemma}\label{lem:psi}
$\mathrm{\Psi}_\epsilon(\hat{\bx} \mid \cD) \ge \mathrm{\Psi}_\epsilon(\bx \mid \cD)$.
\end{lemma}
\begin{proof}
By construction, $\hat{\gamma} \le \epsilon$. Hence, if $(\ba_i, b_i)$  is an inlier to $\bx$, i.e., $|\ba^T_i \bx - b_i| \le \epsilon$, then $|\ba_i^T\hat{\bx} -b_i | \le \hat{\gamma} \le \epsilon$, i.e., $(\ba_i, b_i)$ is also an inlier to $\hat{\bx}$. Thus, the consensus of $\hat{\bx}$ is no smaller than the consensus of $\bx$. \qqed
\end{proof}

Lemmas~\ref{lem:basis} and~\ref{lem:psi} suggest a rudimentary algorithm for consensus maximisation that attempts to find the basis of the maximum consensus set, as encapsulated in the proof of the following theorem.

\begin{theorem}\label{thm:xp}
MAXCON is XP (slice-wise polynomial) in the dimension $d$.
\end{theorem}
\begin{proof}
Let $\bx^*$ be a witness to an instance of MAXCON-D with positive answer, i.e., $\mathrm{\Psi}_\epsilon(\bx^* \mid \cD) \ge \psi$. Let $(\hat{\bx}^*, \hat{\gamma}^*)$ be the minimisers to $\LP[\cC_\epsilon(\bx^* \mid \cD)]$. By Lemma~\ref{lem:psi}, $\hat{\bx}^*$ is also a positive witness to the instance. By Lemma~\ref{lem:basis}, $\hat{\bx}^*$ can be found by enumerating all $(d+1)$-subsets of $\cD$, and solving Chebyshev approximation~\eqref{equ:chebyshev} on each $(d+1)$-subset. There are a total of $\binom{N}{d+1}$ subsets to check; including the time to evaluate $\mathrm{\Psi}_\epsilon(\bx \mid \cD)$ for each candidate, the runtime of this simple algorithm is $\mathcal{O}(N^{d+2}\poly(d))$, which is polynomial in $N$ for a fixed $d$. \qqed
\end{proof}

Theorem~\ref{thm:xp} shows that for a fixed dimension $d$, MAXCON can be solved in time polynomial in the number of measurements $N$ (this is consistent with the results in~\cite{enqvist12,enqvist15}). However, this does not imply that MAXCON is tractable (following the standard meaning of tractability in complexity theory~\cite{garey90,downey99}). Moreover, in practical applications, $d$ could be large (e.g., $d \ge 5$), thus the rudimentary algorithm above will not be efficient for large $N$.

\subsection{W[1]-hard in the dimension}\label{sec:w1hard}

Can we remove $d$ from the exponent of the runtime of a globally optimal algorithm? By establishing W[1]-hardness in the dimension, this section shows that it is not possible. Our proofs are inspired by, but extends quite significantly from, that of~\cite[Section 5]{giannopoulos09}. First, the source problem is as follows.

\begin{problem}[k-CLIQUE]
Given undirected graph $G = (V, E)$ with vertex set $V$ and edge set $E$ and a parameter $k \in \mathbb{N}_+$, does there exist a clique in $G$ with $k$ vertices?
\end{problem}

k-CLIQUE is W[1]-hard w.r.t.~parameter $k$~\cite{wiki:parametrised}. Here, we demonstrate an FPT reduction from k-CLIQUE to MAXCON-D with fixed dimension $d$.

\subsubsection{Generating the input data}

Given input graph $G = (V, E)$, where $V = \{1,\dots,M \}$, and size $k$, we construct a $(k+1)$-dimensional point set $\cD_{G} = \{ (\ba_i, b_i) \}^{N}_{i=1} = \cD_V \cup \cD_E$ as follows:
\begin{itemize}[leftmargin=1em]
\item The set $\cD_V$ is defined as
\begin{align}
\cD_V = \{ (\ba^v_\alpha, b^v_\alpha) \}_{\alpha = 1,\dots,k}^{v = 1, \dots, M},
\end{align}
where
\begin{align}
\ba^v_\alpha = \left[0, \dots, 0, 1, 0, \dots, 0 \right]^T
\end{align}
is a $k$-dimensional vector of $0$'s except at the $\alpha$-th element where the value is $1$, and
\begin{align}
b^v_\alpha = v.
\end{align}
\item The set  $\cD_E$ is defined as
\begin{align}\label{equ:edgeset}
\begin{aligned}
\cD_E = \{ (\ba^{u,v}_{\alpha,\beta}, b^{u,v}_{\alpha,\beta}) \mid~&u, v = 1,\dots,M, \\
&\langle u, v \rangle \in E, \langle v, u \rangle \in E,\\
&\alpha, \beta = 1, \dots, k,\\
&\alpha < \beta~\},
\end{aligned}
\end{align}
where
\begin{align}
\ba^{u,v}_{\alpha,\beta} = \left[0, \dots, 0, 1, 0, \dots, 0, M, 0, \dots, 0 \right]^T
\end{align}
is a $k$-dimensional vector of $0$'s, except at the $\alpha$-th element where the value is $1$ and the $\beta$-th element where the value is $M$, and
\begin{align}
b^{u,v}_{\alpha,\beta} = u + Mv.
\end{align}
\end{itemize}
The size $N$ of $\cD_{G}$ is thus $|\cD_V| + |\cD_E| = kM + 2|E|\binom{k}{2}$.

\subsubsection{Setting the inlier threshold}

Under our reduction, $\bx \in \mathbb{R}^d$ is responsible for ``selecting" a subset of the vertices $V$ and edges $E$ of $G$. First, we say that $\bx$ selects vertex $v$ if a point $(\ba^v_\alpha, b^v_\alpha) \in \cD_V$, for some $\alpha$, is an inlier to $\bx$, i.e., if
\begin{align}\label{eq:VertexInlier}
| (\ba^v_\alpha)^T\bx - b^v_\alpha | \le \epsilon \equiv x_\alpha \in [v - \epsilon, v + \epsilon],
\end{align}
where $x_\alpha$ is the $\alpha$-th element of $\bx$. The key question is how to set the value of the inlier threshold $\epsilon$, such that $\bx$ selects no more than $k$ vertices, or equivalently, such that $\mathrm{\Psi}_\epsilon(\bx \mid \cD_V) \le k$ for all $\bx$.

\begin{lemma}\label{lem:V}
If $\epsilon < \frac{1}{2}$, then $\mathrm{\Psi}_\epsilon(\bx \mid \cD_V) \le k$, with equality achieved if and only if $\bx$ selects $k$ vertices of $G$.
\end{lemma}
\begin{proof}
For any $u$ and $v$, the ranges $[u-\epsilon, u+\epsilon]$ and $[v-\epsilon, v+\epsilon]$ cannot overlap if $\epsilon < \frac{1}{2}$. Hence, $x_\alpha$ lies in at most one of the ranges, i.e., each element of $\bx$ selects at most one of the vertices; see Fig.~\ref{fig:lemma3}. This implies that $\mathrm{\Psi}_\epsilon(\bx \mid \cD_V) \le k$. \qqed
\end{proof}

\begin{figure}[t]\centering 
\subfigure{\includegraphics[width=0.59\textwidth]{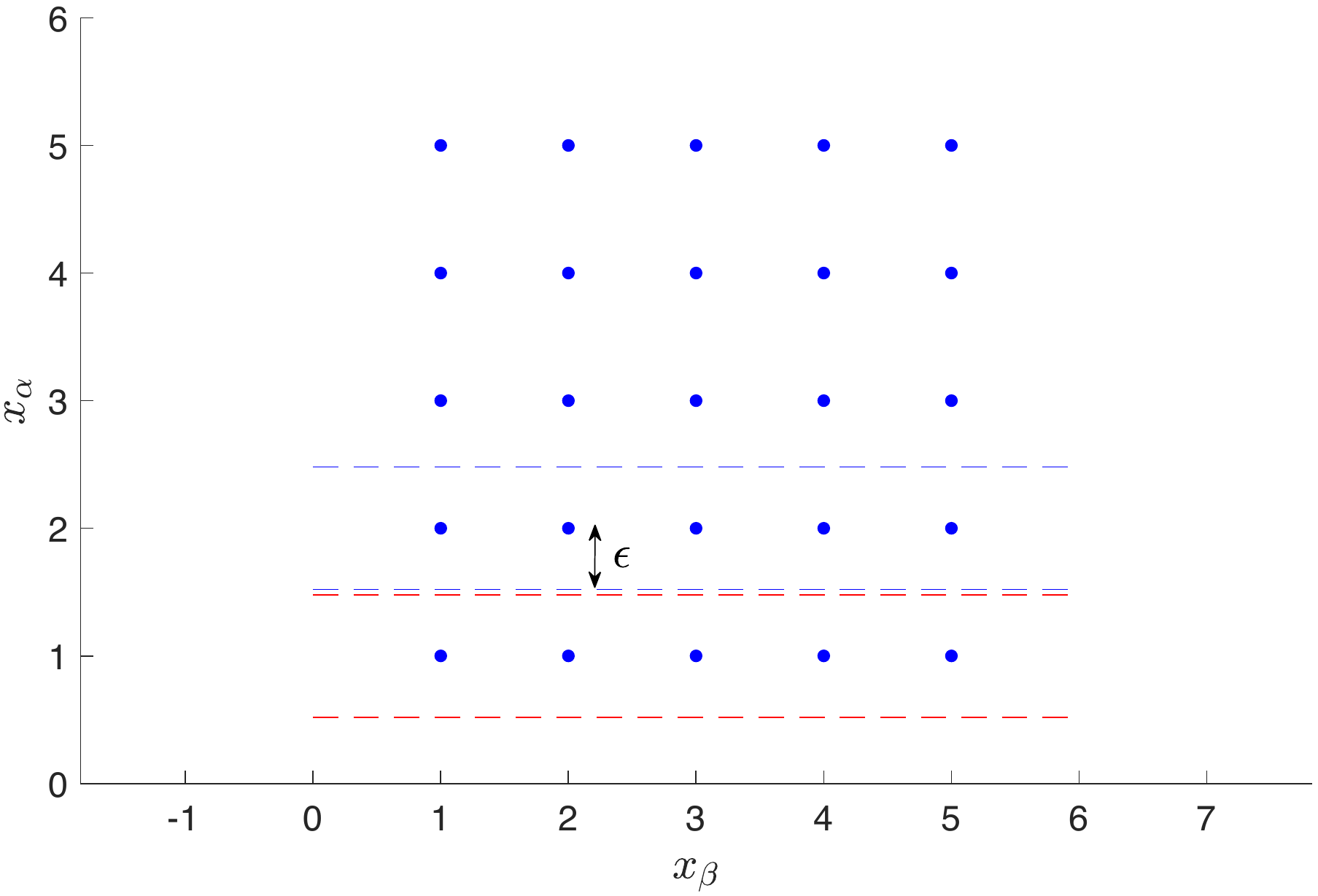}}
\caption{The blue dots indicate the integer values in the dimensions $x_\alpha$ and $x_\beta$. If $\epsilon<\frac{1}{2}$, then the ranges defined by~\eqref{eq:VertexInlier} for all $v = 1,\dots,M$ do not overlap. Hence, $x_\alpha$ can select at most one vertex of the graph.}
\label{fig:lemma3}
\end{figure}

Second, a point $(\ba^{u,v}_{\alpha,\beta}, b^{u,v}_{\alpha,\beta})$ from $\cD_E$ is an inlier to $\bx$ if
\begin{align}\label{equ:edgeinlier}
| (\ba^{u,v}_{\alpha,\beta})^T\bx - b^{u,v}_{\alpha,\beta} | \le \epsilon \equiv | (x_\alpha-u) + M(x_\beta - v)| \le \epsilon.
\end{align}
As suggested by~\eqref{equ:edgeinlier}, the pairs of elements of $\bx$ are responsible for selecting the edges of $G$. To prevent each element pair $x_\alpha, x_\beta$ from selecting more than one edge, or equivalently, to maintain $\mathrm{\Psi}_\epsilon(\bx \mid \cD_E) \le \binom{k}{2}$, the setting of $\epsilon$ is crucial.

\begin{lemma}\label{lem:E}
If $\epsilon < \frac{1}{2}$, then $\mathrm{\Psi}_\epsilon(\bx \mid \cD_E) \le \binom{k}{2}$, with equality achieved if and only if $\bx$ selects $\binom{k}{2}$ edges of $G$.
\end{lemma}
\begin{proof}
For each $\alpha, \beta$ pair, the constraint~\eqref{equ:edgeinlier} is equivalent to the two linear inequalities
\begin{align}\label{equ:oneslab}
\begin{aligned}
x_\alpha + Mx_\beta - u - Mv &\le \epsilon, \\
x_\alpha + Mx_\beta - u - Mv &\ge -\epsilon,
\end{aligned}
\end{align}
which specify two opposing half-planes (i.e., a slab) in the space $(x_\alpha,x_\beta)$.  Note that the slopes of the half-plane boundaries do not depend on $u$ and $v$. For any two unique pairs $(u_1, v_1)$ and $(u_2, v_2)$, we have the four linear inequalities
\begin{align}\label{equ:twoslabs}
\begin{aligned}
x_\alpha + Mx_\beta - u_1 - Mv_1 &\le \epsilon, \\
x_\alpha + Mx_\beta - u_1 - Mv_1 &\ge -\epsilon, \\
x_\alpha + Mx_\beta - u_2 - Mv_2 &\le \epsilon, \\
x_\alpha + Mx_\beta - u_2 - Mv_2 &\ge -\epsilon.
\end{aligned}
\end{align}
The system~\eqref{equ:twoslabs} can be simplified to
\begin{align}\label{equ:twoineq}
\begin{aligned}
\frac{1}{2}\left[ u_2 - u_1 + M(v_2 - v_1) \right] &\le \epsilon, \\
\frac{1}{2}\left[ u_1 - u_2 + M(v_1 - v_2) \right] &\le \epsilon.
\end{aligned}
\end{align}
Setting $\epsilon < \frac{1}{2}$ ensures that the two inequalities~\eqref{equ:twoineq} cannot be consistent for all unique pairs $(u_1, v_1)$ and $(u_2, v_2)$. Geometrically, with $\epsilon < \frac{1}{2}$, the two slabs defined by~\eqref{equ:oneslab} for different $(u_1, v_1)$ and $(u_2, v_2)$ pairs do not intersect; see Fig.~\ref{fig:lemma4} for an illustration.

\begin{figure}[t]\centering 
\subfigure{\includegraphics[width=0.6\columnwidth]{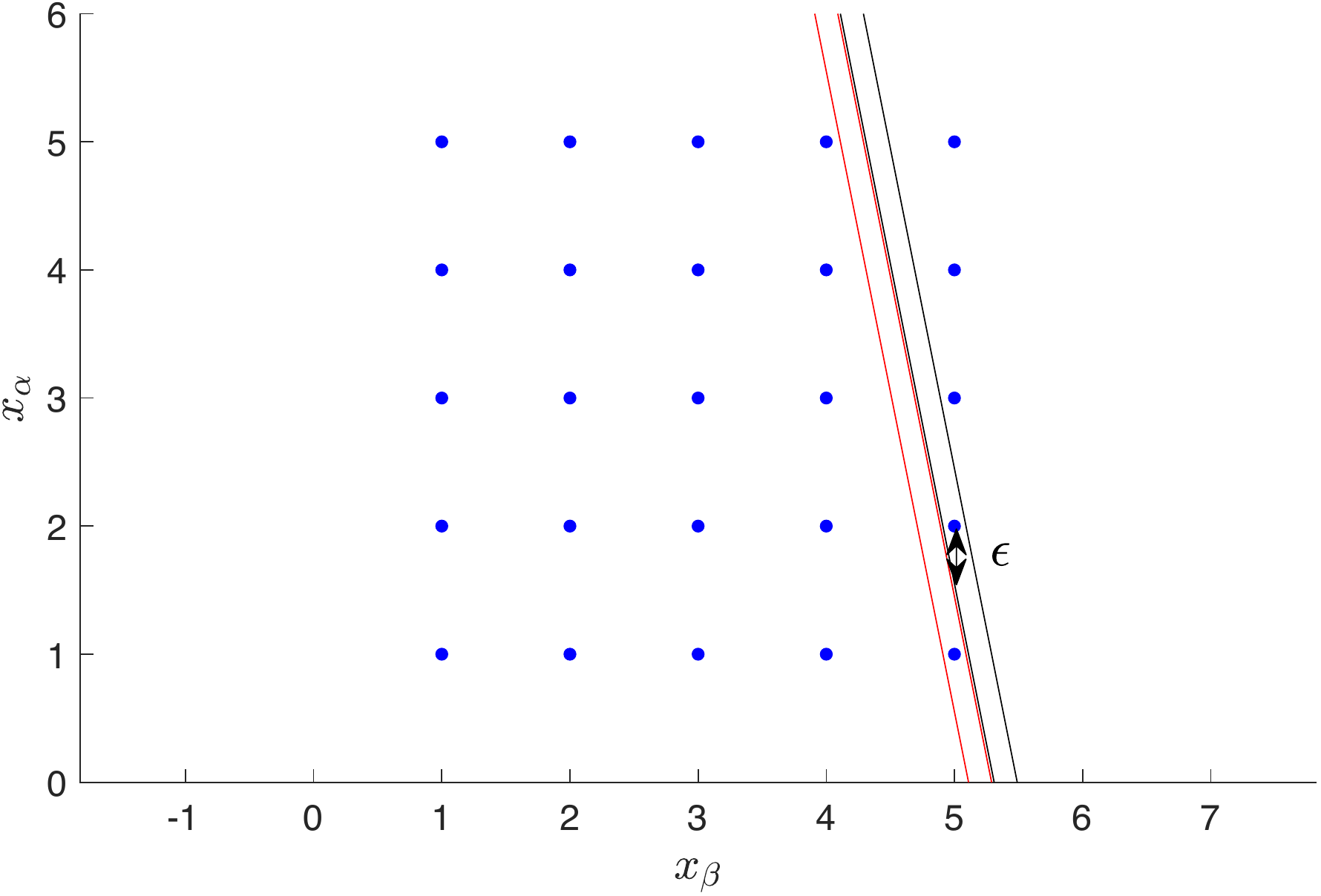}}
\caption{The blue dots indicate the integer values in the dimensions $x_\alpha$ and $x_\beta$. If $\epsilon<\frac{1}{2}$, then any two slabs defined by~\eqref{equ:oneslab} for different $(u_1, v_1)$ and $(u_2, v_2)$ pairs do not intersect. The figure shows two slabs corresponding to $u_1 = 1$, $v_1 = 5$, $u_2 = 2$, $v_2 = 5$.}
\label{fig:lemma4}
\end{figure}

Hence, if $\epsilon < \frac{1}{2}$, each element pair $x_\alpha, x_\beta$ of $\bx$ can select at most one of the edges. Cumulatively, $\bx$ can select at most $\binom{k}{2}$ edges, thus $\mathrm{\Psi}_{\epsilon}(\bx \mid \cD_E) \le \binom{k}{2}$. \qqed
\end{proof}

Up to this stage, we have shown that if $\epsilon < \frac{1}{2}$, then $\mathrm{\Psi}_\epsilon(\bx \mid \cD_G) \le k + \binom{k}{2}$, with equality achievable if there is a clique of size $k$ in $G$. To establish the FPT reduction, we need to establish the reverse direction, i.e., if $\mathrm{\Psi}_\epsilon(\bx \mid \cD_G) = k + \binom{k}{2}$, then there is a $k$-clique in $G$. The following lemma shows that this can be assured by setting $\epsilon<\frac{1}{M+2}$.

\begin{lemma}\label{lem:combine}
If $\epsilon < \frac{1}{M+2}$, then $\mathrm{\Psi}_\epsilon(\bx \mid \cD_G) \le k + \binom{k}{2}$, with equality achievable if and only if there is a clique of size $k$ in $G$.
\end{lemma}
\begin{proof}
The `only if' direction has already been proven. To prove the `if' direction, we show that if $\epsilon<\frac{1}{M+2}$ and $\mathrm{\Psi}_\epsilon(\bx \mid \cD_G) = k + \binom{k}{2}$, the subgraph S($\bx$) = $\{\lfloor x_1 \rceil,...,\lfloor x_k \rceil\}$ is a k-clique, where each $\lfloor x_\alpha \rceil$ represents a vertex index in G. Since $\epsilon<\frac{1}{2}$, $\lfloor x_\alpha \rceil = u$ if and only if $(\ba^u_\alpha, b^u_\alpha)$ is an inlier. Therefore, S($\bx$) consists of all vertices selected by $\bx$. From Lemma~\ref{lem:V} and Lemma~\ref{lem:E}, when $\mathrm{\Psi}_\epsilon(\bx \mid \cD_G) = k + \binom{k}{2}$, $\bx$ is consistent with k points in $\cD_V$ and $\binom{k}{2}$ points in $\cD_E$. The inliers in $\cD_V$ specifies the k vertices in S($\bx$). The `if' direction is true if all selected $\binom{k}{2}$ edges are only edges in S($\bx$), i.e., for each inlier point $(\ba^{u,v}_{\alpha,\beta}, b^{u,v}_{\alpha,\beta})\in\cD_E$, $(\ba^u_\alpha, b^u_\alpha)$ and $(\ba^v_\beta, b^v_\beta)$ are also inliers w.r.t. $\bx$. The prove is done by contradiction:

If $\epsilon<\frac{1}{M+2}$, given an inlier $(\ba^{u,v}_{\alpha,\beta}, b^{u,v}_{\alpha,\beta})$, from~\eqref{equ:edgeinlier} we have:
\begin{align}\label{eq:lem_combine_contradiction}
\begin{aligned}
&| (x_\alpha-u) + M(x_\beta - v)| =\\ 
&|[(\lfloor x_\alpha\rceil -u) + M(\lfloor x_\beta\rceil - v)] + [(x_\alpha-\lfloor x_\alpha\rceil) + M(x_\beta - \lfloor x_\beta\rceil)]| < \frac{1}{M+2}.
\end{aligned}
\end{align}    
Assume at least one of $(\ba^u_\alpha, b^u_\alpha)$ and $(\ba^v_\beta, b^v_\beta)$ is not an inlier, from~\eqref{eq:VertexInlier} and $\epsilon<\frac{1}{M+2}$, we have $\lfloor x_\alpha\rceil \neq u$ or $\lfloor x_\beta\rceil \neq v$, which means that at least one of $(\lfloor x_\alpha\rceil -u)$ and $(\lfloor x_\beta\rceil -v)$ is not zero. Since all elements of $\bx$ satisfy~\eqref{eq:VertexInlier}, both $(\lfloor x_\alpha\rceil -u)$ and $(\lfloor x_\beta\rceil -v)$ are integers between $[-(M-1),(M-1)]$. If only one of $(\lfloor x_\alpha\rceil -u)$ and $(\lfloor x_\beta\rceil -v)$ is not zero, then $|(\lfloor x_\alpha\rceil -u) + M(\lfloor x_\beta\rceil - v)| \geq |1+M\cdot0| = 1$. If both are not zero, then $|(\lfloor x_\alpha\rceil -u) + M(\lfloor x_\beta\rceil - v)| \geq |(M-1)+M\cdot1| = 1$ Therefore, we have
\begin{align}\label{eq:lem_combine1}
|(\lfloor x_\alpha\rceil -u) + M(\lfloor x_\beta\rceil - v)|\geq 1.
\end{align} 
Also due to~\eqref{eq:VertexInlier}, we have
\begin{align}\label{eq:lem_combine2}
|(x_\alpha-\lfloor x_\alpha\rceil) + M(x_\beta - \lfloor x_\beta\rceil)| \leq (M+1)\cdot\epsilon = \frac{M+1}{M+2}. 
\end{align} 
Combining~\eqref{eq:lem_combine1} and~\eqref{eq:lem_combine2}, we have  
\begin{align}
\begin{aligned}
&|[(\lfloor x_\alpha\rceil -u) + M(\lfloor x_\beta\rceil - v)] + [(x_\alpha-\lfloor x_\alpha\rceil) + M(x_\beta - \lfloor x_\beta\rceil)]|\geq\\
&|[(\lfloor x_\alpha\rceil -u) + M(\lfloor x_\beta\rceil - v)]| - |[(x_\alpha-\lfloor x_\alpha\rceil) + M(x_\beta - \lfloor x_\beta\rceil)]|\geq\\
&1-\frac{M+1}{M+2} = \frac{1}{M+2},
\end{aligned}
\end{align} 
which contradicts~\eqref{eq:lem_combine_contradiction}. It is obvious that S($\bx$) can be computed within linear time. Hence, the `if' direction is true when $\epsilon<\frac{1}{M+2}$.\qqed
\end{proof}

To illustrate Lemma~\ref{lem:combine}, Fig.~\ref{fig:lemma5} depicts the value of $\mathrm{\Psi}_\epsilon( \bx \mid \cD_G )$ in the subspace $(x_\alpha, x_\beta)$ for $\epsilon < \frac{1}{M+2}$. Observe that $\mathrm{\Psi}_\epsilon( \bx \mid \cD_G )$ attains the highest value of $3$ in this subspace if and only if $x_\alpha$ and $x_\beta$ select a pair of vertices that are connected by an edge in $G$.

\begin{figure}[t]\centering 
\subfigure{\includegraphics[width=0.6\columnwidth]{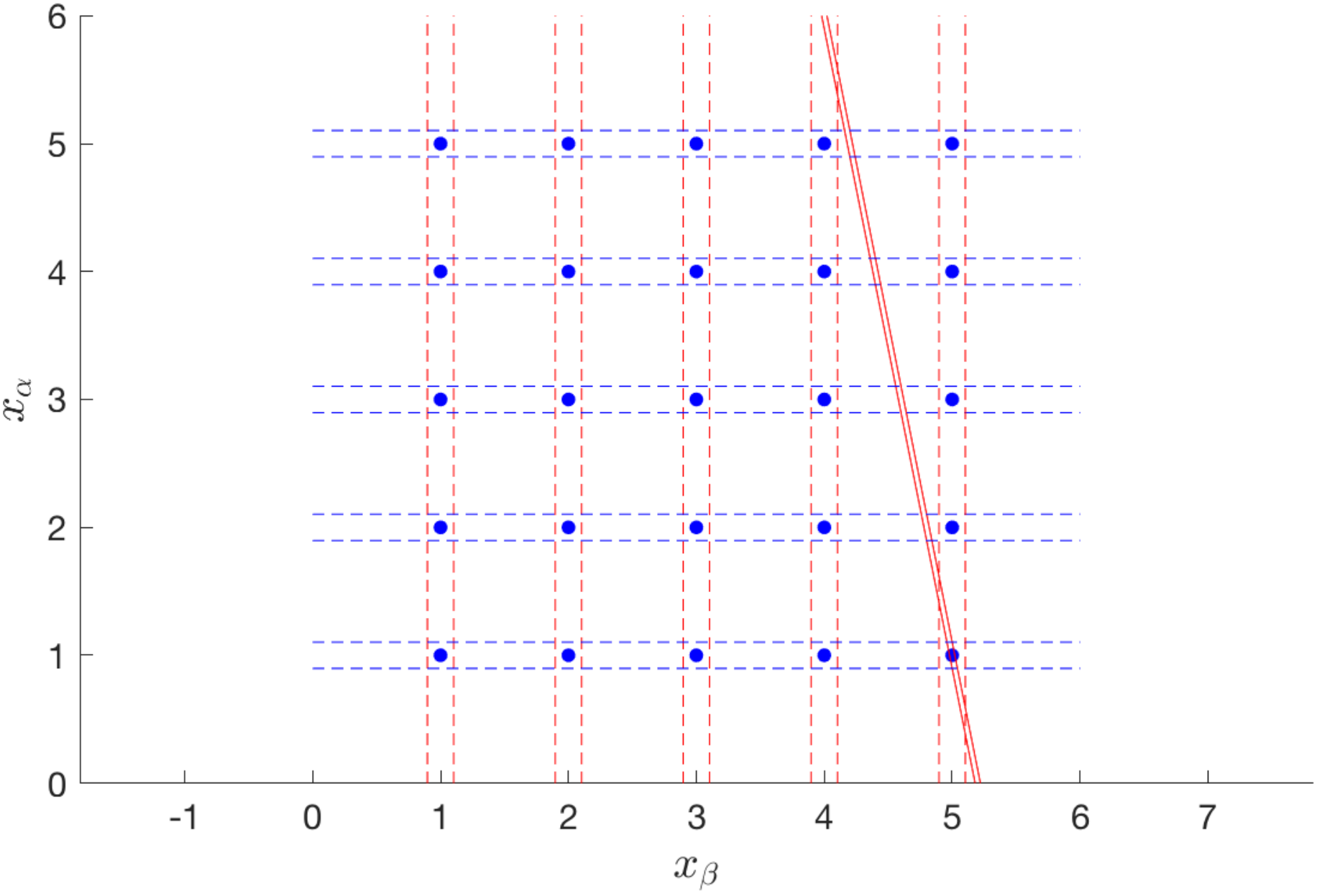}}
\caption{If $\epsilon<\frac{1}{M+2}$, then the slab~\eqref{equ:oneslab} that contains a point $(\ba^{u,v}_{\alpha,\beta}, b^{u,v}_{\alpha,\beta})\in\cD_E$, where $(u,v)$ is an edge in $\cG$, does not intersect with any grid region besides the one formed by $(\ba^u_\alpha, b^u_\alpha)$ and $(\ba^v_\beta, b^v_\beta)$. In this figure, $u = 1$ and $v = 5$.}
\label{fig:lemma5}
\end{figure}

\subsubsection{Completing the reduction}

We have demonstrated a reduction from k-CLIQUE to MAXCON-D, where the main work is to generate data $\cD_G$ which has number of measurements $N = k|V| + 2|E|\binom{k}{2}$ that is linear in $|G|$ and polynomial in $k$, and dimension $d = k$. In other words, the reduction is FPT in $k$. Setting $\epsilon < \frac{1}{M+2}$ and $\psi = k + \binom{k}{2}$ completes the reduction.

\begin{theorem}\label{thm:w1hard}
MAXCON is W[1]-hard w.r.t. the dimension $d$.
\end{theorem}
\begin{proof}
Since k-CLIQUE is W[1]-hard w.r.t.~$k$, by the above FPT reduction, MAXCON is W[1]-hard w.r.t.~$d$. \qqed
\end{proof}

The implications of Theorem~\ref{thm:w1hard} have been discussed in Sec.~\ref{sec:contributions}.

\subsection{FPT in the number of outliers and dimension}\label{sec:fpt}

Let $f(\cC)$ and $\hat{\bx}_\cC$ respectively indicate the minimised objective value and minimiser of $\LP[\cC]$. Consider two subsets $\cP$ and $\cQ$ of $\cD$, where $\cP \subseteq \cQ$. The statement
\begin{align}
f(\cP) \le f(\cQ)
\end{align}
follows from the fact that $\LP[\cP]$ contains only a subset of the constraints of $\LP[\cQ]$; we call this property \emph{monotonicity}.

Let $\bx^*$ be a global solution of an instance of MAXCON, and let $\cI^* := \cC_\epsilon(\bx^* \mid \cD) \subset \cD$ be the maximum consensus set. Let $\cC$ index a subset of $\cD$, and let $\cB$ be the basis of $\cC$. If $f(\cC) > \epsilon$, then by Lemma~\ref{lem:basis}
\begin{align}
f(\cD) \ge f(\cC) =  f(\cB) > \epsilon.
\end{align}
The monotonicity property affords us further insight.

\begin{lemma}\label{lem:trueoutlier}
At least one point in $\cB$ do not exist in $\cI^*$.
\end{lemma}
\begin{proof}
By monotonicity,
\begin{align}\label{equ:monotonicity}
\epsilon < f(\cB) \le  f(\cI^* \cup \cB).
\end{align}
Hence, $\cI^* \cup \cB$ cannot be equal to $\cI^*$, for if they were equal, then $f(\cI^* \cup \cB) = f(\cI^*) \le \epsilon$ which violates~\eqref{equ:monotonicity}. \qqed
\end{proof}

The above observations suggest an algorithm for MAXCON that recursively removes basis points to find a consensus set, as summarised in Algorithm~\ref{alg:tree}. This algorithm is a special case of the technique of Chin et al.~\cite{chin15}. Note that in the worst case, Algorithm~\ref{alg:tree} finds a solution with consensus $d$ (i.e., the minimal case to fit $\bx$), if there are no solutions with higher consensus to be found.

\begin{algorithm}
\begin{algorithmic}[1]
\REQUIRE $\cD = \{ \ba_i, b_i \}^{N}_{i=1}$, threshold $\epsilon$.
\STATE $\cC \leftarrow \{1,\dots,N \}$, $\hat{\bx} \leftarrow$ NULL, $\hat{\psi} \leftarrow 0$.
\STATE $[\hat{\bx}, \hat{\psi}] \leftarrow \mathrm{fitRem(\cC, \epsilon, \hat{\bx}, \hat{\psi})}$.
\RETURN $\hat{\bx}$.
\end{algorithmic}
\vspace{1em}
$[\hat{\bx}, \hat{\psi}] = \mathrm{fitRem(\cC, \epsilon, \hat{\bx}, \hat{\psi})}$
\begin{algorithmic}[1]
\STATE Solve $\LP[\cC]$ to obtain $f(\cC)$, $\hat{\bx}_{\cC}$, and the basis $\cB$ of $\cC$.
\IF{$f(\cC) \le \epsilon$}
\IF{$|\cC| > \hat{\psi}$}
\STATE $\hat{\bx} \leftarrow \hat{\bx}_{\cC}$, $\hat{\psi} \leftarrow |\cC|$.~~//Found a better consensus set.
\ENDIF
\ELSE
\FOR{each $i \in \cB$}
\STATE $[\hat{\bx}, \hat{\psi}] \leftarrow \mathrm{fitRem(\cC\setminus i, \epsilon, \hat{\bx}, \hat{\psi})}$.~~//Remove points from basis and refit.
\ENDFOR
\ENDIF
\RETURN $\hat{\bx}$ and $\hat{\psi}$.
\end{algorithmic}
\caption{FPT algorithm for MAXCON.}\label{alg:tree}
\end{algorithm}

\begin{theorem}\label{thm:fpt}
MAXCON is FPT in the number of outliers and dimension.
\end{theorem}
\begin{proof}
Algorithm~\ref{alg:tree} conducts a depth-first tree search to find a recursive sequence of \emph{basis} points to remove from $\cD$ to yield a consensus set. By Lemma~\ref{lem:trueoutlier}, the longest sequence of basis points that needs to be removed is $o = N-|\cI^\ast|$, which is also the maximum tree depth searched by the algorithm (each descend of the tree removes one point). The number of nodes visited is of order $(d+1)^o$, since the branching factor of the tree is $|\cB|$, and by Lemma~\ref{lem:basis}, $|\cB| \le d+1$.

At each node, $\LP[\cC]$ is solved, with the largest of these LPs having $d+1$ variables and $N$ constraints. Algorithm~\ref{alg:tree} thus runs in $\cO(d^o \mathrm{poly}(N,d))$ time, which is exponential only in the number of outliers $o$ and dimension $d$.\qqed
\end{proof}

Using~\cite[Theorem 2.3]{matousek95} and the repeated basis detection and avoidance procedure in~\cite[Sec.~3.1]{chin15}, the complexity of Algorithm~\ref{alg:tree} can be improved to $\cO((o+1)^d\mathrm{poly}(N,d))$. See~\cite[Sec.~3.5]{chin17} for details.

\section{Approximability}\label{sec:apxhard}

Given the inherent intractability of MAXCON, it is natural to seek recourse in approximate solutions. However, this section shows that it is not possible to construct PTAS~\cite{vazirani01} for MAXCON.

Our development here is inspired by~\cite[Sec.~3.2]{amaldi95}. First, we define our source problem: given a set of $k$ Boolean variables $\{ v_j \}^{k}_{j=1}$, a \emph{literal} is either one of the variables, e.g., $v_j$, or its negation, e.g., $\neg v_j$. A \emph{clause} is a disjunction over a set of literals, i.e., $v_1 \vee \neg v_2 \vee v_3$. A \emph{truth assignment} is a setting of the values of the $k$ variables. A clause is \emph{satisfied} if it evaluates to true.

\begin{problem}[MAX-2SAT]
Given $M$ clauses $\cK = \{ \cK_i \}^{M}_{i=1}$ over $k$ Boolean variables $\{ v_j \}^{k}_{j=1}$, where each clause has exactly two literals, what is the maximum number of clauses that can be satisfied by a truth assignment?
\end{problem}

MAX-2SAT is APX-hard~\cite{wiki:max-2sat}, meaning that there are no algorithms that run in polynomial time that can approximately solve MAX-2SAT up to a desired error ratio. Here, we show an L-reduction~\cite{wiki:l-reduction} from MAX-2SAT to MAXCON, which unfortunately shows that MAXCON is also APX-hard.

\subsubsection{Generating the input data}

Given an instance of MAX-2SAT with clauses $\cK = \{ \cK_i \}^{M}_{i=1}$ over variables $\{ v_j \}^{k}_{j=1}$, let each clause $\cK_i$ be represented as $(\pm v_{\alpha_i})\vee(\pm v_{\beta_i})$, where $\alpha_i, \beta_i \in \{1, \dots, k\}$ index the variables that exist in $\cK_i$, and $\pm$ here indicates either a ``blank" (no negation) or $\neg$ (negation). Define
\begin{align}
\mathrm{sgn}(\alpha_i) = \begin{cases} +1 & \mathrm{if}~v_{\alpha_i}~\textrm{occurs without negation in}~\cK_i, \\ -1 & \mathrm{if}~v_{\alpha_i}~\textrm{occurs with negation in}~\cK_i; \end{cases}
\end{align}
similarly for $\mathrm{sgn}(\beta_i)$. Construct the input data for MAXCON as
\begin{align}
\cD_{\cK} = \{ (\ba^{p}_i,b^{p}_i) \}^{p = 1,\dots,6}_{i=1,\dots,M},
\end{align}
where there are six measurements for each clause. Namely, for each clause $\cK_i$,
\begin{itemize}
\item $\ba^{1}_i$ is a $k$-dimensional vector of zeros, except at the $\alpha_i$-th and $\beta_i$-th elements where the values are respectively $\mathrm{sgn}(\alpha_i)$ and $\mathrm{sgn}(\beta_i)$, and $b^{1}_i = 2$.
\item $\ba^{2}_i = \ba^{1}_i$ and $b^{2}_i = 0$.
\item $\ba^{3}_i$ is a $k$-dimensional vector of zeros, except at the $\alpha_i$-th element where the value is $\mathrm{sgn}(\alpha_i)$, and $b^{3}_i = -1$.
\item $\ba^{4}_i = \ba^{3}_i$ and $b^{4}_i = 1$.
\item $\ba^{5}_i$ is a $k$-dimensional vector of zeros, except at the $\beta_i$-th element where the value is $\mathrm{sgn}(\beta_i)$, and $b^{5}_i = -1$.
\item $\ba^{6}_i = \ba^{5}_i$ and $b^{6}_i = 1$.
\end{itemize}
The number of measurements $N$ in $\cD_{\cK}$ is $6M$.

\subsubsection{Setting the inlier threshold}

Given a solution $\bx \in \mathbb{R}^k$ for MAXCON, the six input measurements associated with $\cK_i$ are inliers under these conditions:
\begin{align}\label{equ:pair1}
\begin{aligned}
(\ba^1_i, b^1_i)~\textrm{is an inlier} &\iff |\mathrm{sgn}(\alpha_i)x_{\alpha_i} + \mathrm{sgn}(\beta_i)x_{\beta_i} - 2| \le \epsilon,\\
(\ba^2_i, b^2_i)~\textrm{is an inlier} &\iff |\mathrm{sgn}(\alpha_i)x_{\alpha_i} + \mathrm{sgn}(\beta_i)x_{\beta_i}| \le \epsilon,
\end{aligned}
\end{align}
\begin{align}\label{equ:pair2}
\begin{aligned}
(\ba^3_i, b^3_i)~\textrm{is an inlier} &\iff |\mathrm{sgn}(\alpha_i)x_{\alpha_i} + 1| \le \epsilon,\\
(\ba^4_i, b^4_i)~\textrm{is an inlier} &\iff |\mathrm{sgn}(\alpha_i)x_{\alpha_i} - 1| \le \epsilon,
\end{aligned}
\end{align}
\begin{align}\label{equ:pair3}
\begin{aligned}
(\ba^5_i, b^5_i)~\textrm{is an inlier} &\iff |\mathrm{sgn}(\beta_i)x_{\beta_i} + 1| \le \epsilon,\\
(\ba^6_i, b^6_i)~\textrm{is an inlier} &\iff |\mathrm{sgn}(\beta_i)x_{\beta_i} - 1| \le \epsilon,
\end{aligned}
\end{align}
where $x_{\alpha}$ is the $\alpha$-th element of $\bx$. Observe that if $\epsilon < 1$, then at most one of~\eqref{equ:pair1}, one of~\eqref{equ:pair2}, and one of~\eqref{equ:pair3} can be satisfied. The following result establishes an important condition for L-reduction.

\begin{lemma}
If $\epsilon < 1$, then
\begin{align}\label{equ:lreduc1}
\mathrm{OPT(\textrm{MAXCON})} \le 6\cdot \mathrm{OPT(\textrm{MAX-2SAT})},
\end{align}
$\mathrm{OPT(\textrm{MAX-2SAT})}$ is the maximum number of clauses that can be satisfied for a given MAX-2SAT instance, and  $\mathrm{OPT(\textrm{MAXCON})}$ is the maximum achievable consensus for the MAXCON instance generated under our reduction.
\end{lemma}
\begin{proof}
If $\epsilon < 1$, for all $\bx$, at most one of~\eqref{equ:pair1}, one of~\eqref{equ:pair2}, and one~\eqref{equ:pair3}, can be satisfied, hence $\mathrm{OPT(\textrm{MAXCON})}$ cannot be greater than $3M$. For any MAX-2SAT instance with $M$ clauses, there is an algorithm~\cite{johnson74} that can satisfy at least $\lceil \frac{M}{2} \rceil$ of the clauses, thus $\mathrm{OPT(\textrm{MAX-2SAT})} \ge \lceil \frac{M}{2} \rceil$. This leads to~\eqref{equ:lreduc1}.\qqed
\end{proof}

Note that, if $\epsilon < 1$, rounding $\bx$ to its nearest bipolar vector (i.e,, a vector that contains only $-1$ or $1$) cannot decrease the consensus w.r.t.~$\cD_\cK$. It is thus sufficient to consider $\bx$ that are bipolar in the rest of this section.

Intuitively, $\bx$ is used as a proxy for truth assignment: setting $x_j = 1$ implies setting $v_j = true$, and vice versa. Further, if one of the conditions in~\eqref{equ:pair1} holds for a given $\bx$, then the clause $\cK_i$ is satisfied by the truth assignment. Hence, for $\bx$ that is bipolar and $\epsilon < 1$,
\begin{align}
\mathrm{\Psi}_\epsilon( \bx \mid \cD_\cK ) = 2M + \sigma, 
\end{align}
where $\sigma$ is the number of clauses satisfied by $\bx$. This leads to the final necessary condition for L-reduction.

\begin{lemma}
If $\epsilon < 1$, then
\begin{align}\label{equ:lreduc2}
\left| \mathrm{OPT(\textrm{MAX-2SAT})} - \mathrm{SAT}(\bt(\bx)) \right| = \left| \mathrm{OPT(\textrm{MAXCON})} - \mathrm{\Psi}_\epsilon(\bx \mid \cD_\cK) \right|,
\end{align}
where $\bt(\bx)$ returns the truth assignment corresponding to $\bx$, and $\mathrm{SAT}(\bt(\bx))$ returns the number of clauses satisfied by $\bt(\bx)$.
\end{lemma}
\begin{proof}
For any bipolar $\bx$ with consensus $2M + \sigma$, the truth assignment $\bt(\bx)$ satisfies exactly $\sigma$ clauses. Since the value of $\mathrm{OPT(\textrm{MAXCON})}$ must take the form $2M + \sigma^*$, then $\mathrm{OPT(\textrm{MAX-2SAT})} = \sigma^*$. The condition~\eqref{equ:lreduc2} is immediately seen to hold by substituting the values into the equation.\qqed
\end{proof}

We have demonstrated an L-reduction from MAX-2SAT to MAXCON, where the main work is to generate $\cD_\cK$ in linear time. The function $\bt$ also takes linear time to compute. Setting $\epsilon < 1$ completes the reduction.

\begin{theorem}\label{thm:apxhard}
MAXCON is APX-hard.
\end{theorem}
\begin{proof}
Since MAX-2SAT is APX-hard, by the above L-reduction, MAXCON is also APX-hard.\qqed
\end{proof}

See Sec.~\ref{sec:contributions} for the implications of Theorem~\ref{thm:apxhard}.

\section{Conclusions and future work}

Given the fundamental difficulty of consensus maximisation as implied by our results (see Sec.~\ref{sec:contributions}), it would be prudent to consider alternative paradigms for optimisation, e.g., deterministically convergent heuristic algorithms~\cite{le17,purkait17,cai18} or preprocessing techniques~\cite{svarm14,parrabustos15,chin16}.

\section*{Acknowledgements}

This work was supported by ARC Grant DP160103490.

\clearpage

\bibliographystyle{splncs}
\bibliography{final}
\end{document}